\documentclass[11pt]{article}

\usepackage[preprint]{acl}

\usepackage{times}
\usepackage{latexsym}

\usepackage[T1]{fontenc}

\usepackage[utf8]{inputenc}

\usepackage{microtype}

\usepackage{inconsolata}

\usepackage{graphicx}

\usepackage{url}
\usepackage{booktabs}
\usepackage{nicefrac}
\usepackage{xcolor}
\usepackage{fancyvrb}
\usepackage{fvextra}

\usepackage{wrapfig}
\usepackage{tikz}
\usepackage{float}
\usepackage{caption}
\usepackage{enumerate}
\usepackage{array}
\usepackage[ruled, linesnumbered, boxed]{algorithm2e}
\usepackage{todonotes}

\usepackage{multirow}
\usepackage{makecell}
\usepackage{bbding}
\usepackage{mathtools}

\usepackage[nameinlink]{cleveref}
\Crefname{figure}{Figure}{Figures}
\crefname{figure}{Figure}{Figures}
\crefname{example}{Example}{Example}
\crefname{theorem}{Theorem}{Theorem}
\crefname{corollary}{Corollary}{Corollary}
\crefname{lemma}{Lemma}{Lemma}
\crefname{proposition}{Proposition}{Proposition}
\crefname{assumption}{Assumption}{Assumption}
\crefname{section}{Section}{Section}
\crefname{algorithm}{Algorithm}{Algorithm}

\usepackage{amsthm,thmtools}
\declaretheorem[name=Theorem,numberwithin=section]{theorem}

\declaretheorem[name=Proposition,numberlike=theorem]{proposition}

\declaretheorem[name=Remark,style=definition,numberwithin=section]{remark}

\usepackage{hyperref}
\usepackage{pgfplots}
\pgfplotsset{compat=1.17}
\usepackage{subfig}
\usepackage{amssymb}

\newcommand{\R}{\mathbb{R}}

\newcommand{\calP}{\mathcal{P}}

\newcommand{\calE}{\mathcal{E}}

\newcommand{\abs}[1]{\left|#1\right|}

\usepackage{amsmath,amsfonts,bm}

\def\eqref#1{equation~\ref{#1}}

\def\1{\bm{1}}

\DeclareMathAlphabet{\mathsfit}{\encodingdefault}{\sfdefault}{m}{sl}
\SetMathAlphabet{\mathsfit}{bold}{\encodingdefault}{\sfdefault}{bx}{n}

\providecommand{\R}{\mathbb{R}}

\usepackage{mdframed}

\title{Thought-Aware KV Cache Compaction for Reasoning via Adaptive Attention Matching}

\author{%
  Yang Liu$^{1}$, Bin Chong$^{2}$\thanks{Corresponding author.}, Chongyang Zhang$^{3}$, Hao Zheng$^{3}$, Jiayu Liang$^{4}$, Xu Kefu$^{2}$\\
  $^1$Tsinghua University
  $^2$Peking University
  $^3$Fullive.AI
  $^4$Soochow University
}

\begin{document}
\maketitle
\begin{abstract}
Reasoning language models generate lengthy chain-of-thought (CoT) sequences whose key-value (KV) cache grows linearly and becomes a memory bottleneck during decoding. Existing compaction methods treat reasoning trajectories as flat token sequences and apply uniform compression, ignoring the hierarchical structure of CoT reasoning where different steps vary drastically in importance. We propose \textbf{Thought-Aware Attention Matching (TAM)}, which exploits this structure through three mechanisms: (i)~thought segmentation that decomposes the trajectory into reasoning blocks, (ii)~adaptive budget allocation that assigns compression budget based on each segment's importance and size, and (iii)~pivotal token protection that preserves high-attention reasoning anchors. We prove that the allocation rule is optimal under a convex error model and that cumulative error under sequential compaction remains bounded. Experiments on AIME 2024 and MATH-500 with Qwen3-4B show that TAM improves accuracy over uniform compaction at the same memory footprint, with periodic compaction bounding peak memory to 3.1--3.2\,GB (a 65\% reduction) while maintaining competitive accuracy.
\end{abstract}

\section{Introduction}
\label{sec:introduction}

Memory has emerged as a critical bottleneck in modern language models deployed for long-horizon tasks. Reasoning-capable models produce lengthy chain-of-thought (CoT) sequences~\citep{wei2022cot} before final answers, and the key-value (KV) cache that stores attention states for all previous tokens grows linearly with sequence length~\citep{zhang2023h2o}. This growth substantially increases memory footprint and reduces decoding throughput~\citep{dao2022flashattention}, limiting the practical deployment of reasoning models on resource-constrained devices. While existing KV cache compression methods~\citep{zhang2023h2o,li2024snapkv} primarily target long input prompts at prefill time, the efficient compaction of \emph{generated} tokens during decoding, i.e.\ mid-trajectory compaction, remains underexplored despite its importance for reasoning workloads that routinely produce long chains of intermediate tokens.

Recent work has made progress along complementary directions. Attention Matching (AM)~\citep{zweiger2026fastkv} achieves large one-shot compaction by constructing compact keys and values that reproduce attention outputs and preserve attention mass, using closed-form solutions that avoid gradient descent. However, AM targets prefill scenarios and assumes reference queries from self-study on a fixed context. Reasoning Path Compression (RPC)~\citep{song2025rpc} addresses mid-trajectory compression by periodically evicting tokens based on a selector window of recent queries. Yet RPC relies on simple eviction, which, as we show formally, systematically underestimates attention mass at high compression ratios. Both approaches treat the reasoning trajectory as a flat token sequence and apply uniform compression, ignoring the hierarchical structure of CoT reasoning. While prior methods vary along two axes---the compaction primitive (eviction vs.\ optimization) and the query source (self-study vs.\ selector window)---neither exploits a third, orthogonal axis: structure-aware budget allocation.

CoT trajectories~\citep{kojima2022zeroshotcot} naturally decompose into distinct reasoning steps---problem restatements, intermediate calculations, exploratory dead ends, and conclusions---whose importance to future generation varies dramatically. Dead-end explorations become irrelevant as reasoning progresses, while key intermediate results and problem definitions remain critical throughout. This heterogeneity motivates non-uniform budget allocation: by assigning more compression budget to important reasoning steps and less to irrelevant ones, we can preserve the tokens that matter most while aggressively compressing the rest.

We propose \textbf{Thought-Aware Attention Matching (TAM)}, which realizes this idea through three mechanisms: (i)~thought segmentation that decomposes the reasoning trajectory into coherent blocks, (ii)~adaptive budget allocation that distributes compact keys based on each segment's importance and size, and (iii)~pivotal token protection that preserves high-attention reasoning anchors. These mechanisms operate on top of the AM optimization pipeline and a selector window for lightweight query generation, transforming uniform compaction into a structure-aware scheme.

We make three contributions:
\begin{enumerate}
    \item We introduce TAM, a structure-aware mid-trajectory KV cache compaction method that segments reasoning trajectories into thought blocks, allocates compression budget adaptively based on each segment's importance and size, and protects critical reasoning anchors via pivotal token identification.
    \item We provide theoretical foundations: the allocation $t_i \propto \sqrt{w_i \cdot n_i}$ is optimal under convex error models (Proposition~\ref{prop:allocation}), pivotal protection reduces approximation error (Proposition~\ref{prop:pivotal}), and cumulative error under sequential compaction remains bounded (Proposition~\ref{prop:cumulative}).
    \item We evaluate TAM on AIME 2024 and MATH-500 with Qwen3-4B, demonstrating consistent accuracy improvements over uniform compaction at the same memory budget, with periodic compaction achieving up to 65\% peak-memory reduction.
\end{enumerate}

\section{Related Work}
\label{sec:related}

\paragraph{KV cache eviction and merging.}
KV cache compression has been studied extensively for long-context inference~\citep{xiao2024streamingllm,bai2023longbench,liu2023lost,ding2023longnet}. H2O~\citep{zhang2023h2o} retains tokens with the highest accumulated attention under a sliding window; SnapKV~\citep{li2024snapkv} selects important KV positions by observing attention patterns within an observation window. These methods target long input prompts at prefill time and are not designed for dynamically growing context during decoding. All eviction-based methods apply uniform retention policies, treating every token position equivalently.

\paragraph{Latent-space compaction.}
Cartridges~\citep{eyuboglu2025cartridges} train compact KV caches offline via prefix-tuning at the cost of GPU-hours per context. Attention Matching (AM)~\citep{zweiger2026fastkv} decomposes the compaction problem into closed-form subproblems (OMP for key selection, NNLS for biases, OLS for value fitting), matching Cartridges quality at far lower cost but assuming one-shot compaction at prefill time. Both methods apply uniform compression across all tokens.

\paragraph{Mid-trajectory compression for reasoning.}
Reasoning Path Compression (RPC)~\citep{song2025rpc} addresses mid-trajectory compression by periodically evicting tokens based on a selector window of recent queries. While RPC identifies the selector window as a lightweight proxy for future queries, it relies on eviction rather than latent-space optimization and applies a uniform retention threshold across the entire prefix. TAM operates along an orthogonal axis: it allocates compression budget non-uniformly across reasoning segments based on their importance, a design dimension independent of the compaction primitive or query source.

\section{Preliminaries}
\label{sec:preliminaries}

\subsection{Attention and KV Cache}
\label{sub:attention}

In the Transformer architecture~\citep{vaswani2017attention}, each layer computes attention over queries $\mathbf{Q} \in \R^{n \times d}$, keys $\mathbf{K} \in \R^{T \times d}$, and values $\mathbf{V} \in \R^{T \times d}$. The scaled dot-product attention is
\begin{equation}
\text{Attn}(\mathbf{Q}, \mathbf{K}, \mathbf{V}) = \text{softmax}\left(\frac{\mathbf{Q}\mathbf{K}^\top}{\sqrt{d}}\right) \mathbf{V},
\end{equation}
where the softmax is applied row-wise. For autoregressive generation, the model caches keys and values for all $T$ context tokens, appending new KV pairs as each token is generated. The KV cache thus grows linearly with sequence length, becoming a memory bottleneck for long reasoning trajectories~\citep{dao2022flashattention}.

A key insight for compaction is that attention over concatenated KV blocks decomposes into a mixture of each block's locally normalized output, weighted by that block's attention mass. For a key block $\mathbf{K}$, define the \emph{attention mass} as
\begin{equation}
\text{Mass}(\mathbf{q}; \mathbf{K}) = \sum_{j=1}^{T} \exp\left(\frac{\mathbf{q} \mathbf{K}_j^\top}{\sqrt{d}}\right).
\end{equation}
Matching both the local attention output and the attention mass of a compacted block suffices to preserve its contribution when concatenated with arbitrary future tokens~\citep{zweiger2026fastkv}.

\subsection{Attention Matching (AM)}
\label{sub:am}

Attention Matching~\citep{zweiger2026fastkv} replaces the original cache $(\mathbf{K}, \mathbf{V})$ with $\mathbf{K}, \mathbf{V} \in \R^{T \times d}$ by a compact representation $(\mathbf{C}_k, \boldsymbol{\beta}, \mathbf{C}_v)$ where $\mathbf{C}_k, \mathbf{C}_v \in \R^{t \times d}$ with $t \ll T$, and $\boldsymbol{\beta} \in \R^t$ is a per-token scalar bias. The compact cache should satisfy, for reference queries $\mathbf{q}_i$ of interest,
\begin{equation}
\small
\frac{\exp(\mathbf{q}_i \mathbf{K}^\top / \sqrt{d}) \mathbf{V}}{\sum_j \exp(\mathbf{q}_i \mathbf{K}_j^\top / \sqrt{d})} \approx \frac{\exp(\mathbf{q}_i \mathbf{C}_k^\top / \sqrt{d} + \boldsymbol{\beta})}{\sum_j \exp(\mathbf{q}_i (\mathbf{C}_k)_j^\top / \sqrt{d} + \beta_j)} \mathbf{C}_v,
\end{equation}
\begin{equation}
\small
\sum_{j=1}^{T} \exp(\mathbf{q}_i \mathbf{K}_j^\top / \sqrt{d}) \approx \sum_{j=1}^{t} \exp(\mathbf{q}_i (\mathbf{C}_k)_j^\top / \sqrt{d} + \beta_j).
\end{equation}
The first condition matches the local attention output; the second matches the attention mass. With the substitution $u_j = \exp(\beta_j) \geq 0$, the mass condition becomes a nonnegative least squares problem over $\mathbf{u}$, solved via NNLS~\citep{lawson1995solving}. Given $\mathbf{C}_k$ and $\boldsymbol{\beta}$, the value matrix $\mathbf{C}_v$ is fitted via ordinary least squares. Key selection is performed via orthogonal matching pursuit~\citep{tropp2007omp} (which greedily selects keys to minimize mass residual) or by retaining keys with highest aggregate attention under the reference queries. In prefill AM, reference queries are obtained from self-study or repeat-prefill on the fixed context.

\section{Thought-Aware Attention Matching (TAM)}
\label{sec:tam}

\subsection{Overview}
\label{sub:motivation}

TAM transforms uniform compaction into a structure-aware scheme through three mechanisms operating on top of the AM optimization pipeline: (i)~thought segmentation, (ii)~adaptive budget allocation, and (iii)~pivotal token protection. Section~\ref{sub:theory} provides theoretical justification for the allocation strategy; Figure~\ref{fig:tam-overview} illustrates the full pipeline.

\begin{figure*}[t]
\centering
\includegraphics[width=\linewidth]{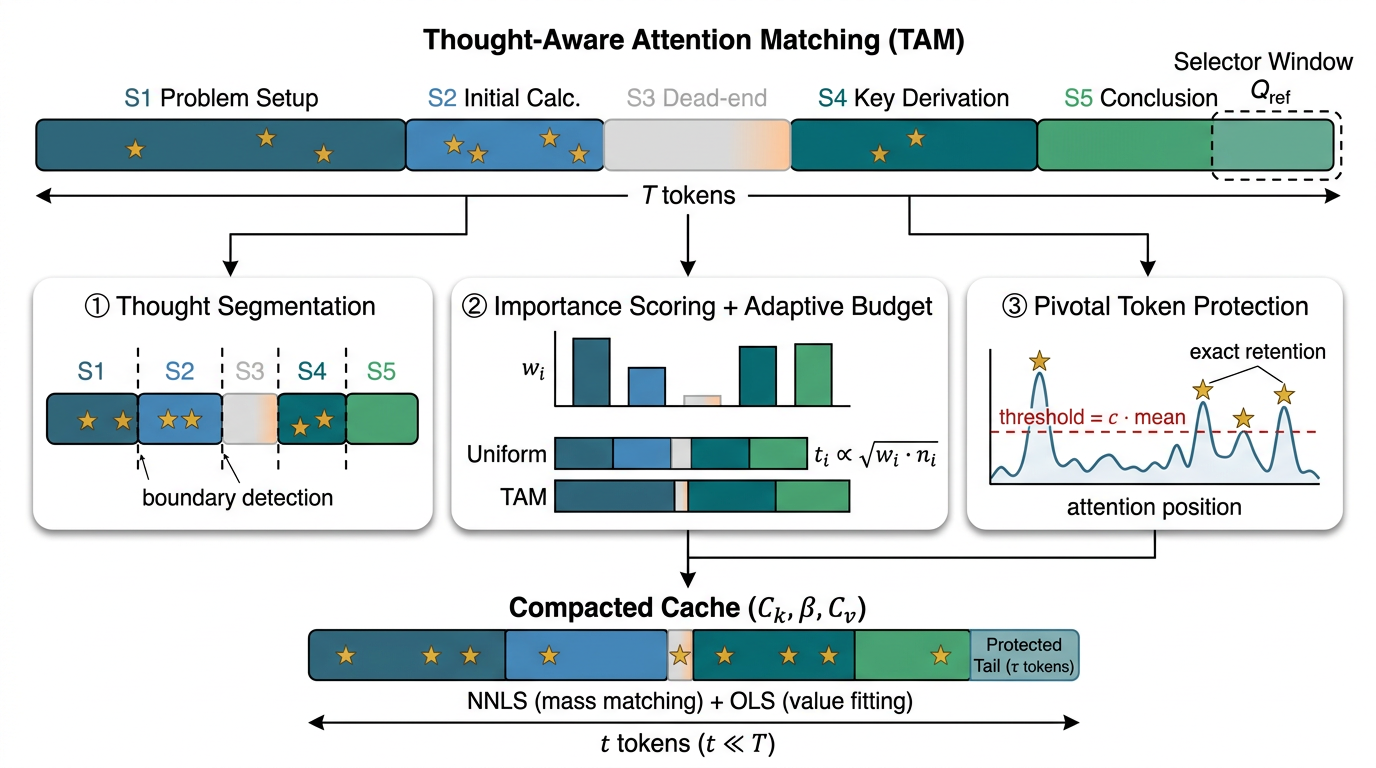}
\caption{Overview of the TAM pipeline. 
}
\label{fig:tam-overview}
\end{figure*}

\subsection{Selector Window Query Generation}
\label{sub:selector}

Let the current sequence length be $L$, comprising the prompt and all reasoning tokens generated so far. The \textbf{selector window} comprises the last $R$ tokens. We extract query vectors $\mathbf{Q}_{\text{ref}} \in \R^{R \times d}$ from these positions by running a single forward pass with hooks to capture query activations. This yields per-layer, per-head reference queries; TAM processes each KV-head independently using its corresponding $\mathbf{Q}_{\text{ref}}$. These queries encode the current reasoning frontier and serve as a lightweight proxy for future query distributions without requiring self-study or repeat-prefill~\citep{song2025rpc}.

\subsection{Thought Segmentation}
\label{sub:segmentation}

TAM decomposes the compactable prefix into reasoning segments $\{S_1, S_2, \ldots, S_m\}$, where each segment corresponds to a coherent reasoning step.

\paragraph{Heuristic segmentation.}
Chain-of-thought outputs exhibit clear structural patterns that delineate reasoning boundaries, most prominently paragraph breaks (double newlines). Reasoning models are trained to organize their output into coherent steps separated by such structural markers, so these boundaries reliably correspond to semantic transitions. We detect double-newline boundaries in the decoded token text using a lightweight rule-based scanner that runs in $O(L)$ time. Consecutive short segments (fewer than $\ell_{\min}$ tokens) are merged to avoid over-fragmentation.

\paragraph{Attention-based segmentation (alternative).}
We also evaluate an attention-based approach that detects boundaries from shifts in attention intensity. For each position $j$, we compute the mean attention $\bar{\alpha}_j$ it receives from the selector window (defined formally in Eq.~\ref{eq:pivotal}). We then measure the local gradient $\Delta_j = |\bar{\alpha}_j - \bar{\alpha}_{j-1}|$ and place boundaries at positions where $\Delta_j$ exceeds $2\times$ the median gradient. This detects transitions between high-attention and low-attention regions, which empirically correspond to reasoning step boundaries. We compare both approaches in Section~\ref{sub:ablations}.

\begin{table*}[t]
\centering
\begin{minipage}[t]{0.54\linewidth}
  \centering
  \caption{Main results at target compaction ratio 0.1 (retain 10\% of KV entries). ``Mem'' is post-compaction steady-state GPU memory (see Section~\ref{sub:setup}); for one-shot methods ($P{=}\infty$), peak memory during generation equals No Compaction. Only TAM (Periodic) maintains bounded memory throughout. Best accuracy in \textbf{bold}; best memory in \underline{underline}.}
  \label{tab:main}
  \small
  \setlength{\tabcolsep}{4pt}
  \resizebox{\linewidth}{!}{%
  \begin{tabular}{@{}l*{4}{r}@{}}
  \toprule
  \multirow{2}{*}{\textbf{Method}} & \multicolumn{2}{c}{\textbf{AIME 2024} (30 prob.)} & \multicolumn{2}{c}{\textbf{MATH-500}} \\
  \cmidrule(lr){2-3} \cmidrule(lr){4-5}
  & Acc (\%) & Mem (GB) & Acc (\%) & Mem (GB) \\
  \midrule
  No Compaction & 63.3 & 9.2 & 71.2 & 8.8 \\
  \addlinespace[2pt]
  Eviction (Selector Window) & 46.7 & 4.1 & 52.4 & 3.9 \\
  AM + Repeat & 53.3 & 4.2 & 61.0 & 4.0 \\
  PAM (Uniform) & 56.7 & 4.0 & 64.6 & 3.8 \\
  \addlinespace[2pt]
  \textbf{TAM (Ours)} & \textbf{60.0} & 4.0 & \textbf{67.8} & 3.8 \\
  TAM (Periodic, $P{=}1024$) & 56.7 & \underline{3.2} & 65.4 & \underline{3.1} \\
  \bottomrule
  \end{tabular}%
  }
\end{minipage}
\hfill
\begin{minipage}[t]{0.44\linewidth}
  \centering
  \caption{TAM compaction time breakdown per step (Qwen3-4B, $\sim$4k tokens). TAM-specific stages are marked with $\dagger$.}
  \label{tab:time-breakdown}
  \small
  \resizebox{\linewidth}{!}{%
  \begin{tabular}{@{}lr@{}}
  \toprule
  \textbf{Stage} & \textbf{Time (s)} \\
  \midrule
  Query extract (selector window) & 0.8 \\
  Thought segmentation + importance$^\dagger$ & 0.1 \\
  Pivotal token identification$^\dagger$ & $<$0.1 \\
  Per-segment key selection & 2.5 \\
  $\beta$ fitting (NNLS, global) & 1.2 \\
  Value fitting (OLS, global) & 0.9 \\
  \midrule
  \textbf{Total} & \textbf{5.6} \\
  \quad\emph{of which TAM-specific} & \emph{$\sim$0.15} \\
  \bottomrule
  \end{tabular}%
  }
\end{minipage}
\end{table*}

\begin{figure*}[t]
\centering
\subfloat[Accuracy vs.\ compaction ratio (AIME 2024).]{
\begin{tikzpicture}
\begin{axis}[
  width=0.48\linewidth, height=4.8cm,
  xlabel={Target ratio (fraction retained)},
  ylabel={Accuracy (\%)},
  xtick={0.05,0.1,0.2},
  xticklabels={0.05,0.10,0.20},
  ymin=28, ymax=68,
  legend style={at={(0.98,0.02)},anchor=south east,font=\tiny,draw=gray!50},
  grid=major, grid style={gray!20},
  every axis plot/.append style={thick,mark size=2.5pt},
]
\addplot[dashed, gray, forget plot] coordinates {(0.05,63.3)(0.2,63.3)};
\node[font=\tiny,gray] at (axis cs:0.15,64.8) {No Compaction};
\addplot[red!70!black, mark=triangle*] coordinates {(0.05,33.3)(0.1,46.7)(0.2,53.3)};
\addplot[orange!80!black, mark=diamond*] coordinates {(0.05,46.7)(0.1,53.3)(0.2,56.7)};
\addplot[blue!70!black, mark=o] coordinates {(0.05,50.0)(0.1,56.7)(0.2,60.0)};
\addplot[green!60!black, mark=star, mark size=3.5pt, line width=1.2pt] coordinates {(0.05,53.3)(0.1,60.0)(0.2,63.3)};
\addplot[violet, mark=square*] coordinates {(0.05,50.0)(0.1,56.7)(0.2,60.0)};
\legend{Eviction,AM+Repeat,PAM (Uniform),TAM (Ours),TAM (Periodic)}
\end{axis}
\end{tikzpicture}
\label{fig:accuracy-vs-ratio}
}
\hfill
\subfloat[Accuracy vs.\ post-compaction memory (AIME 2024).]{
\begin{tikzpicture}
\begin{axis}[
  width=0.48\linewidth, height=4.8cm,
  xlabel={Post-compaction Memory (GB)},
  ylabel={Accuracy (\%)},
  ymin=28, ymax=68,
  x dir=reverse,
  legend style={at={(0.02,0.02)},anchor=south west,font=\tiny,draw=gray!50},
  grid=major, grid style={gray!20},
  every axis plot/.append style={thick,mark size=2.5pt},
]
\addplot[only marks, gray, mark=square*, mark size=3pt] coordinates {(9.2,63.3)};
\addlegendentry{No Compaction}
\addplot[red!70!black, mark=triangle*] coordinates {(3.0,33.3)(4.1,46.7)(5.7,53.3)};
\addplot[orange!80!black, mark=diamond*] coordinates {(3.1,46.7)(4.2,53.3)(5.8,56.7)};
\addplot[blue!70!black, mark=o] coordinates {(3.0,50.0)(4.0,56.7)(5.6,60.0)};
\addplot[green!60!black, mark=star, mark size=3.5pt, line width=1.2pt] coordinates {(3.0,53.3)(4.0,60.0)(5.6,63.3)};
\addplot[violet, mark=square*] coordinates {(2.4,50.0)(3.2,56.7)(4.5,60.0)};
\legend{No Compaction,Eviction,AM+Repeat,PAM (Uniform),TAM (Ours),TAM (Periodic)}
\end{axis}
\end{tikzpicture}
\label{fig:accuracy-vs-memory}
}
\caption{Accuracy vs.\ compaction trade-offs on AIME 2024 (Qwen3-4B). Left: accuracy vs.\ compaction ratio. Right: accuracy vs.\ post-compaction memory. TAM consistently outperforms all baselines; TAM (Periodic) achieves the best memory vs.\ accuracy trade-off.}
\label{fig:accuracy-tradeoffs}
\end{figure*}
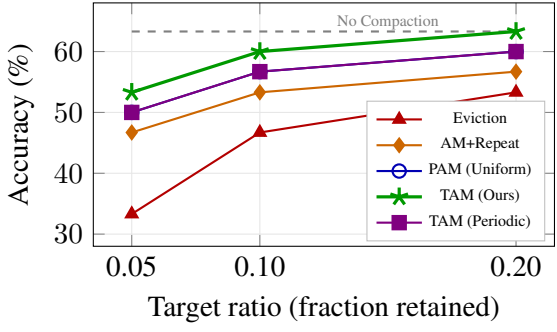
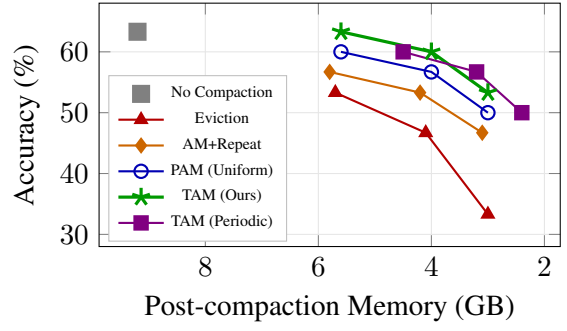

\begin{figure*}[t]
\centering
\subfloat[Accuracy vs.\ compaction ratio (MATH-500).]{
\begin{tikzpicture}
\begin{axis}[
  width=0.48\linewidth, height=4.8cm,
  xlabel={Target ratio (fraction retained)},
  ylabel={Accuracy (\%)},
  xtick={0.05,0.1,0.2},
  xticklabels={0.05,0.10,0.20},
  ymin=30, ymax=76,
  legend style={at={(0.98,0.02)},anchor=south east,font=\tiny,draw=gray!50},
  grid=major, grid style={gray!20},
  every axis plot/.append style={thick,mark size=2.5pt},
]
\addplot[dashed, gray, forget plot] coordinates {(0.05,71.2)(0.2,71.2)};
\node[font=\tiny,gray] at (axis cs:0.15,72.7) {No Compaction};
\addplot[red!70!black, mark=triangle*] coordinates {(0.05,38.4)(0.1,52.4)(0.2,61.8)};
\addplot[orange!80!black, mark=diamond*] coordinates {(0.05,49.2)(0.1,61.0)(0.2,66.2)};
\addplot[blue!70!black, mark=o] coordinates {(0.05,52.6)(0.1,64.6)(0.2,68.4)};
\addplot[green!60!black, mark=star, mark size=3.5pt, line width=1.2pt] coordinates {(0.05,56.0)(0.1,67.8)(0.2,70.2)};
\addplot[violet, mark=square*] coordinates {(0.05,53.2)(0.1,65.4)(0.2,68.0)};
\legend{Eviction,AM+Repeat,PAM (Uniform),TAM (Ours),TAM (Periodic)}
\end{axis}
\end{tikzpicture}
\label{fig:accuracy-vs-ratio-math}
}
\hfill
\subfloat[Accuracy vs.\ post-compaction memory (MATH-500).]{
\begin{tikzpicture}
\begin{axis}[
  width=0.48\linewidth, height=4.8cm,
  xlabel={Post-compaction Memory (GB)},
  ylabel={Accuracy (\%)},
  ymin=30, ymax=76,
  x dir=reverse,
  legend style={at={(0.02,0.02)},anchor=south west,font=\tiny,draw=gray!50},
  grid=major, grid style={gray!20},
  every axis plot/.append style={thick,mark size=2.5pt},
]
\addplot[only marks, gray, mark=square*, mark size=3pt] coordinates {(8.8,71.2)};
\addlegendentry{No Compaction}
\addplot[red!70!black, mark=triangle*] coordinates {(2.9,38.4)(3.9,52.4)(5.4,61.8)};
\addplot[orange!80!black, mark=diamond*] coordinates {(3.0,49.2)(4.0,61.0)(5.5,66.2)};
\addplot[blue!70!black, mark=o] coordinates {(2.9,52.6)(3.8,64.6)(5.3,68.4)};
\addplot[green!60!black, mark=star, mark size=3.5pt, line width=1.2pt] coordinates {(2.9,56.0)(3.8,67.8)(5.3,70.2)};
\addplot[violet, mark=square*] coordinates {(2.3,53.2)(3.1,65.4)(4.3,68.0)};
\legend{No Compaction,Eviction,AM+Repeat,PAM (Uniform),TAM (Ours),TAM (Periodic)}
\end{axis}
\end{tikzpicture}
\label{fig:accuracy-vs-memory-math}
}
\caption{Accuracy vs.\ compaction trade-offs on MATH-500 (Qwen3-4B, 500 problems). Left: accuracy vs.\ compaction ratio. Right: accuracy vs.\ post-compaction memory. With larger sample size, TAM's advantage over PAM is more clearly separated. TAM achieves near-full-accuracy recovery at ratio 0.2 (70.2\% vs.\ 71.2\%).}
\label{fig:accuracy-tradeoffs-math}
\end{figure*}
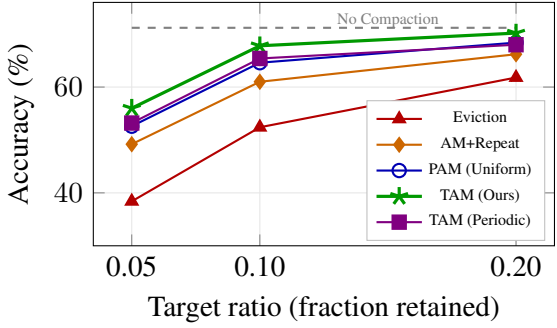
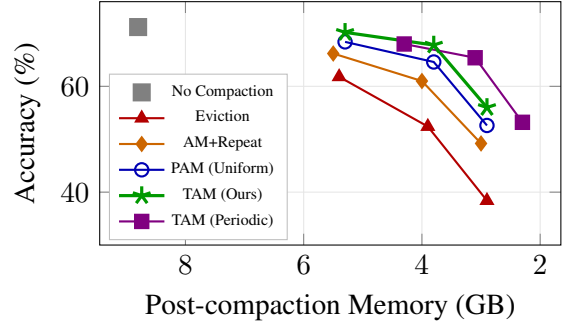

\begin{table*}[t]
\centering
\begin{minipage}[t]{0.33\linewidth}
  \centering
  \caption{Component ablation at target ratio 0.1. 
  }
  \label{tab:ablation-components}
  \small
  \setlength{\tabcolsep}{4.5pt}
  \resizebox{\linewidth}{!}{%
  \begin{tabular}{@{}lrrrr@{}}
  \toprule
  \multirow{2}{*}{\textbf{Configuration}} & \multicolumn{2}{c}{\textbf{AIME 2024}} & \multicolumn{2}{c}{\textbf{MATH-500}} \\
  \cmidrule(lr){2-3} \cmidrule(lr){4-5}
  & Acc (\%) & $\Delta$ & Acc (\%) & $\Delta$ \\
  \midrule
  PAM (Uniform) & 56.7 & 0 & 64.6 & 0 \\
  \quad + Adaptive allocation & 56.7 & +0.0 & 66.4 & +1.8 \\
  \quad + Pivotal protection & 56.7 & +0.0 & 66.2 & +1.6 \\
  \quad + Both (\textbf{TAM}) & \textbf{60.0} & \textbf{+3.3} & \textbf{67.8} & \textbf{+3.2} \\
  \bottomrule
  \end{tabular}%
  }
\end{minipage}
\hfill
\begin{minipage}[t]{0.63\linewidth}
  \centering
  \caption{Ablation Study (TAM, target ratio 0.1, AIME 2024).}
  \label{tab:ablation-seg}
  \small
  \resizebox{\linewidth}{!}{%
  \begin{tabular}{@{}lrr@{}}
  \toprule
  \textbf{Segmentation Method} & \textbf{Acc (\%)} & \textbf{Seg.\ Time (ms)} \\
  \midrule
  Heuristic (double-newline) & \textbf{60.0} & 2 \\
  Attention-based (gradient jump) & 60.0 & 45 \\
  Fixed-length ($\ell{=}256$) & 56.7 & 1 \\
  \bottomrule
  \end{tabular}%
  }
\end{minipage}
\end{table*}

\begin{table*}[t]
\centering
\begin{minipage}[t]{0.48\linewidth}
  \centering
  \caption{Ablation on pivotal threshold multiplier $c$ (TAM, target ratio 0.1, AIME 2024).}
  \label{tab:ablation-pivotal}
  \small
  \begin{tabular}{@{}crr@{}}
  \toprule
  \textbf{Threshold $c$} & \textbf{$|\calP|$ (avg)} & \textbf{Acc (\%)} \\
  \midrule
  1 & 312 & 56.7 \\
  \textbf{3} (default) & 87 & \textbf{60.0} \\
  5 & 34 & 60.0 \\
  10 & 8 & 56.7 \\
  \bottomrule
  \end{tabular}%
\end{minipage}
\hfill
\begin{minipage}[t]{0.48\linewidth}
  \centering
  \caption{Ablation on compaction interval $P$ (TAM, target ratio 0.1, AIME 2024).}
  \label{tab:ablation-p}
  \small
  \begin{tabular}{@{}crr@{}}
  \toprule
  \textbf{Interval $P$} & \textbf{Acc (\%)} & \textbf{Peak Mem (GB)} \\
  \midrule
  $\infty$ (max only) & \textbf{60.0} & 4.0 \\
  \textbf{1024} (default) & 56.7 & 3.2 \\
  512 & 56.7 & 3.0 \\
  256 & 53.3 & \textbf{2.8} \\
  \bottomrule
  \end{tabular}%
\end{minipage}
\end{table*}

\subsection{Segment Importance Scoring and Adaptive Budget Allocation}
\label{sub:allocation}

Given segments $\{S_1, \ldots, S_m\}$ with sizes $n_1, \ldots, n_m$ and selector window queries $\mathbf{Q}_{\text{ref}}$, we compute the \emph{importance} of each segment as the average attention mass it receives from the current reasoning frontier:
\begin{equation}
\label{eq:importance}
w_i = \frac{1}{|\mathbf{Q}_{\text{ref}}|} \sum_{\mathbf{q} \in \mathbf{Q}_{\text{ref}}} \sum_{j \in S_i} \alpha_j(\mathbf{q}),
\end{equation}
where $\alpha_j(\mathbf{q})$ is the softmax attention weight of token $j$ under query $\mathbf{q}$ (normalized over the full prefix including the protected tail). Since the sum only covers compactable positions, $\sum_i w_i \leq 1$; the residual is the tail's attention mass, which is typically small.

\paragraph{Optimal budget allocation.}
Let $\calP$ denote the set of pivotal tokens (Section~\ref{sub:pivotal}) and $t' = t - |\calP|$ the remaining budget after reserving their slots. We allocate $t_i$ compact keys to segment $S_i$ according to:
\begin{equation}
\label{eq:allocation}
t_i = \max\!\left(1,\; \left\lfloor t' \cdot \frac{\sqrt{w_i \cdot n_i}}{\sum_{j=1}^{m} \sqrt{w_j \cdot n_j}} \right\rfloor\right).
\end{equation}
Any residual budget from rounding is assigned to the segment with the largest fractional remainder. This allocation has a principled justification: in Proposition~\ref{prop:allocation}, we prove that when the per-segment approximation error is a convex function of the local compression ratio $n_i / t_i$, the importance-weighted allocation $t_i \propto \sqrt{w_i \cdot n_i}$ minimizes the total weighted approximation error under a linear error model. The allocation naturally concentrates budget on segments that are both large and important, while assigning minimal budget to irrelevant segments.

Standard AM and its uniform variant with selector window queries (PAM, for Prefix Attention Matching) use $t_i \propto n_i$, which wastes budget on low-importance segments. Pure attention-based allocation $t_i \propto w_i$ ignores segment size, potentially under-representing large segments. TAM's $\sqrt{w_i \cdot n_i}$ allocation balances both factors optimally.

\subsection{Pivotal Token Protection}
\label{sub:pivotal}

Beyond the fixed tail protection of the last $\tau$ tokens, TAM identifies and protects \emph{pivotal tokens}: positions that receive consistently high attention from the selector window, indicating their role as reasoning anchors (e.g., problem constants, key intermediate results, critical definitions).

For each token $j$ in the compactable prefix, we compute its average attention score:
\begin{equation}
\label{eq:pivotal}
\bar{\alpha}_j = \frac{1}{|\mathbf{Q}_{\text{ref}}|} \sum_{\mathbf{q} \in \mathbf{Q}_{\text{ref}}} \alpha_j(\mathbf{q}).
\end{equation}
A token is pivotal if $\bar{\alpha}_j > \delta$, where we set $\delta = c \cdot \bar{\alpha}_{\text{mean}}$ with $c = 3$ by default, selecting tokens whose attention significantly exceeds the mean. The pivotal set $\calP$ is constructed by collecting all such tokens; their keys and values are directly retained in the compact cache without modification. The remaining budget $t' = t - |\calP|$ is distributed across segments via Eq.~\ref{eq:allocation}, and key selection within each segment operates only on non-pivotal tokens.

In Proposition~\ref{prop:pivotal}, we show that pivotal protection reduces the worst-case approximation error by a factor proportional to the total attention mass of pivotal tokens.

\subsection{Periodic Compaction Trigger}
\label{sub:periodic}

Compaction is triggered when the cache reaches the maximum sequence length or $P$ new tokens have been generated since the last compaction. At each trigger, TAM re-runs the full pipeline on the current cache: segment boundaries are recomputed from scratch on the updated token sequence, importance scores are recalculated using a fresh selector window, and pivotal tokens are re-identified. Previously protected pivotal tokens receive no special treatment in subsequent rounds---they compete for retention based on their current attention scores. The target size $t$ remains fixed across rounds, so the cache oscillates between $t$ and $t + P$ tokens. The parameter $P$ trades off peak memory against compaction overhead: smaller $P$ compacts more frequently, keeping memory bounded at the cost of higher compute; larger $P$ amortizes overhead but allows higher peak memory. Proposition~\ref{prop:cumulative} quantifies how $P$ controls cumulative error.

\subsection{Algorithm}
\label{sub:algorithm}

Algorithm~\ref{alg:tam} summarizes one TAM compaction step. After protecting the last $\tau$ tokens and extracting $\mathbf{Q}_{\text{ref}}$, TAM segments the compactable prefix, identifies pivotal tokens, allocates per-segment budgets, and runs key selection within each segment. While key selection is per-segment, the bias fitting (NNLS) and value fitting (OLS) are performed \emph{globally} across all selected keys, ensuring that attention mass and outputs are matched at the full-prefix level. To guarantee exact retention of pivotal tokens (as required by Proposition~\ref{prop:pivotal}), we fix $\beta_j = 0$ for all pivotal keys $j \in \calP$, so their original logits are preserved unchanged; NNLS only optimizes biases for non-pivotal selected keys.

\subsection{Theoretical Analysis}
\label{sub:theory}

We provide principled justification for TAM's design choices. Proofs are deferred to Appendix~\ref{app:proofs}.


\begin{proposition}[Mass Deficit of Eviction]
\label{prop:eviction}
Consider eviction-based compaction that retains a subset $S \subset [T]$ with $|S| = t < T$ keys and values without scalar biases. For any query $\mathbf{q}$, define the \emph{evicted mass fraction}
\begin{equation}
\mu(\mathbf{q}) = \frac{\sum_{j \notin S} \exp(\mathbf{q}\mathbf{K}_j^\top / \sqrt{d})}{\sum_{j=1}^{T} \exp(\mathbf{q}\mathbf{K}_j^\top / \sqrt{d})} = \sum_{j \notin S} \alpha_j(\mathbf{q}),
\end{equation}
where $\alpha_j(\mathbf{q})$ is the original attention weight of token $j$. Then:
\begin{enumerate}
    \item $\text{Mass}(\mathbf{q}; \mathbf{C}_k) = (1 - \mu(\mathbf{q})) \cdot \text{Mass}(\mathbf{q}; \mathbf{K}) < \text{Mass}(\mathbf{q}; \mathbf{K})$ for all $\mathbf{q}$ whenever $S \neq [T]$.
    \item When the compacted block is concatenated with future KV pairs $(\mathbf{K}_f, \mathbf{V}_f)$, the compacted prefix's contribution to the softmax denominator is deflated by factor $(1 - \mu(\mathbf{q}))$, shifting attention toward $(\mathbf{K}_f, \mathbf{V}_f)$.
    \item Considering the prefix block in isolation, the attention output error satisfies $\|\hat{\mathbf{y}} - \mathbf{y}\| \leq 2\mu(\mathbf{q}) \cdot \|\mathbf{V}\|_\infty$, where $\|\mathbf{V}\|_\infty = \max_j \|\mathbf{V}_j\|$.
\end{enumerate}
\end{proposition}

Proposition~\ref{prop:eviction} formalizes the fundamental limitation of eviction: the mass deficit $\mu(\mathbf{q})$ equals the total attention weight of evicted tokens, which can be substantial even when each individual evicted token has low attention (the ``long tail'' of small weights sums to a non-negligible fraction). TAM addresses this by fitting scalar biases $\boldsymbol{\beta}$ via NNLS to preserve mass for the reference queries.

\section{Experiments}
\label{sec:experiments}

We evaluate TAM on mathematical reasoning benchmarks, comparing against uniform compaction baselines and eviction methods. We measure downstream accuracy, peak memory usage, and compaction time, and provide ablation studies on TAM's key components.

\subsection{Experimental Setup}
\label{sub:setup}

\paragraph{Benchmarks and model.}
We evaluate on two benchmarks: (1)~\textbf{AIME 2024}~\citep{aime2024} (AIME~I and AIME~II, 30 problems total), which requires multi-step mathematical reasoning and produces lengthy CoT trajectories (median generation length $\approx$4k tokens); and (2)~\textbf{MATH-500}, a 500-problem subset of the MATH test set~\citep{hendrycks2021math,cobbe2021training} spanning algebra, geometry, number theory, and combinatorics, providing larger-scale evaluation with greater statistical power. Each problem has a deterministic answer; we use Qwen3-4B~\citep{qwen3_2025} as the base model with greedy decoding (temperature 0) and a maximum generation length of 8192 tokens. Accuracy is measured by exact match (pass@1)~\citep{chen2021evaluating}.


\paragraph{Compared methods.}
We compare six configurations: (1) \textbf{No Compaction}: full KV cache (accuracy upper bound); (2) \textbf{Eviction (Selector Window)}: retains high-attention keys and evicts the rest~\citep{song2025rpc}; (3) \textbf{AM + Repeat}: one-shot AM at max length with repeat-prefill queries~\citep{zweiger2026fastkv}; (4) \textbf{PAM (Uniform)}: AM with selector window queries and uniform budget (no segmentation or pivotal protection), compacting once at max length ($P=\infty$); (5) \textbf{TAM (Ours)}: AM with selector window queries, segmentation, adaptive budget, and pivotal protection, compacting once at max length ($P=\infty$); (6) \textbf{TAM (Periodic)}: TAM with periodic triggers every $P=1024$ tokens for bounded peak memory.

We note that H2O~\citep{zhang2023h2o} and SnapKV~\citep{li2024snapkv} are prefill-oriented eviction methods designed for long input contexts. Our Eviction (Selector Window) baseline adapts their core mechanism---attention-based token selection---to the mid-trajectory setting, providing a representative comparison for eviction-based approaches in the decoding scenario.

\paragraph{Evaluation protocol.}
We sweep the target compaction ratio over $\{0.05, 0.1, 0.2\}$ (retain 5\%, 10\%, or 20\% of KV entries). We protect the last $\tau{=}20$ tokens, use a selector window of $R{=}64$ tokens, and set the pivotal threshold multiplier $c{=}3$. Thought segmentation uses the heuristic (double-newline) method by default with $\ell_{\min}{=}32$. Key selection uses the highest-attention heuristic. All experiments use a single NVIDIA A100 GPU (80GB). The model is loaded in 4-bit quantization (GPTQ-Int4), so model weights occupy $\approx$2.5\,GB; the reported ``Mem (GB)'' is the \emph{post-compaction steady-state memory} measured via \texttt{torch.cuda.max\_memory\_allocated} after the first compaction step completes, averaged over all problems. For one-shot methods ($P{=}\infty$), this reflects memory after the single compaction at max sequence length; the transient peak before compaction is comparable to No Compaction. For TAM (Periodic), memory stays bounded throughout generation.

\subsection{Main Results}
\label{sub:main-results}

Table~\ref{tab:main} reports results at target ratio 0.1 on both benchmarks. Figure~\ref{fig:accuracy-vs-ratio} and Figure~\ref{fig:accuracy-vs-ratio-math} sweep the compaction ratio for AIME and MATH-500 respectively; Figure~\ref{fig:accuracy-vs-memory} and Figure~\ref{fig:accuracy-vs-memory-math} plot the accuracy vs.\ memory trade-off. We highlight five findings.

First, the Eviction baseline suffers a substantial accuracy drop on both benchmarks (46.7\% on AIME, 52.4\% on MATH-500 vs.\ 63.3\% and 71.2\%), confirming Proposition~\ref{prop:eviction}: mass deficit degrades generation quality. Second, PAM (Uniform) improves over eviction by applying the AM pipeline but still falls short of No Compaction. Third, \textbf{TAM consistently outperforms PAM (Uniform)} on both benchmarks (60.0\% vs.\ 56.7\% on AIME; 67.8\% vs.\ 64.6\% on MATH-500 at ratio 0.1), demonstrating that thought-aware allocation and pivotal protection provide genuine gains over uniform compression. Fourth, TAM (Periodic) achieves the best steady-state memory (3.2\,GB, 65\% reduction vs.\ No Compaction) while maintaining competitive accuracy. On MATH-500, TAM (Periodic) outperforms PAM by 0.6--0.8 points at ratios 0.05--0.1, confirming that TAM's per-step advantage partially survives cumulative error from repeated compaction. At ratio 0.2, cumulative error slightly offsets TAM's advantage on MATH-500 (68.0\% vs.\ 68.4\% for PAM), consistent with diminishing returns of adaptive allocation when ample budget reduces the impact of importance heterogeneity. Fifth, at ratio 0.2, TAM recovers the full No Compaction accuracy on AIME (63.3\%); on MATH-500, a modest 1.0-point gap remains (70.2\% vs.\ 71.2\%).

\paragraph{Statistical significance.} AIME 2024 has only 30 problems (3.3\% per-problem granularity), so individual accuracy differences (e.g., TAM vs.\ PAM: 60.0\% vs.\ 56.7\%, a single-problem difference) are not statistically significant by themselves and should be interpreted with caution. In particular, TAM (Periodic) and PAM are statistically indistinguishable on AIME at all ratios; MATH-500 provides the meaningful comparison for periodic compaction. On MATH-500, TAM vs.\ PAM (67.8\% vs.\ 64.6\%, $\Delta = 16$ problems) yields a one-sided McNemar test $p \approx 0.04$, providing moderate evidence of improvement. The consistent direction of improvement across both benchmarks and all three compression ratios (Table~\ref{tab:full-accuracy-vs-ratio} and Table~\ref{tab:full-accuracy-vs-ratio-math}) strengthens the case despite AIME's limited statistical power. We recommend that future work use larger evaluation sets and report confidence intervals; we note this limitation in Section~\ref{sec:conclusion}.

\begin{figure*}[t]
\centering
\begin{minipage}[t]{0.45\linewidth}
  \centering
  \begin{tikzpicture}
  \begin{axis}[
    ybar,
    width=\linewidth, height=4.6cm,
    xlabel={Compaction Interval $P$},
    ylabel={Accuracy (\%)},
    symbolic x coords={256,512,1024,$\infty$},
    xtick=data,
    ymin=48, ymax=65,
    bar width=10pt,
    nodes near coords,
    nodes near coords style={font=\tiny,anchor=south},
    grid=major, grid style={gray!20},
    axis y line*=left,
    legend style={at={(0.02,0.98)},anchor=north west,font=\tiny,draw=gray!50},
  ]
  \addplot[fill=green!50!black!60, draw=green!50!black!80] coordinates {(256,53.3) (512,56.7) (1024,56.7) ($\infty$,60.0)};
  \addlegendentry{Accuracy}
  \end{axis}
  \begin{axis}[
    ybar,
    width=\linewidth, height=4.6cm,
    symbolic x coords={256,512,1024,$\infty$},
    xtick=\empty,
    ylabel={Peak Memory (GB)},
    ymin=2.0, ymax=5.0,
    bar width=10pt,
    axis y line*=right, axis x line=none,
    legend style={at={(0.98,0.98)},anchor=north east,font=\tiny,draw=gray!50},
    bar shift=12pt,
    nodes near coords,
    nodes near coords style={font=\tiny,anchor=south},
  ]
  \addplot[fill=blue!50, draw=blue!70] coordinates {(256,2.8) (512,3.0) (1024,3.2) ($\infty$,4.0)};
  \addlegendentry{Peak Memory}
  \end{axis}
  \end{tikzpicture}
  \caption{Ablation on compaction interval $P$ (AIME 2024). Smaller $P$ reduces peak memory with modest accuracy trade-off. $P{=}\infty$ compacts only at max sequence length.}
  \label{fig:ablation-p}
\end{minipage}
\hfill
\begin{minipage}[t]{0.51\linewidth}
  \centering
  \begin{tikzpicture}
  \begin{axis}[
    ybar,
    width=\linewidth, height=4.6cm,
    ylabel={Importance $w_i$ (\%)},
    xlabel={Reasoning Segment},
    symbolic x coords={S1,S2,S3,S4,S5,S6,S7,S8,S9,S10,S11,S12},
    xtick=data,
    x tick label style={font=\tiny, rotate=30, anchor=east},
    ymin=0, ymax=20,
    bar width=6pt,
    grid=major, grid style={gray!20},
    nodes near coords,
    nodes near coords style={font=\tiny,anchor=south},
    point meta=explicit symbolic,
    enlarge x limits=0.06,
    legend style={at={(0.98,0.98)},anchor=north east,font=\tiny,draw=gray!50},
  ]
  \addplot[fill=green!60!black!60, draw=green!60!black!80] coordinates {
    (S1,14.2) [14.2] (S2,9.8) [9.8] (S3,6.5) [6.5]
    (S4,2.2) [\textcolor{red!70!black}{2.2}]
    (S5,1.8) [\textcolor{red!70!black}{1.8}]
    (S6,13.5) [13.5] (S7,8.8) [8.8]
    (S8,1.5) [\textcolor{red!70!black}{1.5}]
    (S9,7.2) [7.2] (S10,10.5) [10.5] (S11,6.8) [6.8] (S12,17.2) [17.2]
  };
  \addplot[dashed, gray, thick, no markers, forget plot] coordinates {(S1,8.33)(S12,8.33)};
  \node[font=\tiny,gray] at (axis cs:S11,9.5) {uniform};
  \end{axis}
  \end{tikzpicture}
  \caption{Per-segment importance $w_i$ for a representative AIME problem (12 segments). Importance varies by $>11\times$ (S12: 17.2\% vs.\ S8: 1.5\%). Dead-end segments (S4, S5, S8, red labels) receive minimal budget under TAM's adaptive allocation.}
  \label{fig:segment-importance}
\end{minipage}
\end{figure*}

\subsection{Compaction Time Breakdown}
\label{sub:compaction-time}

Table~\ref{tab:time-breakdown} profiles TAM's wall-clock cost per compaction step on a typical $\sim$4k-token trajectory. Thought segmentation and importance scoring add minimal overhead (0.1\,s total) relative to the AM core (key selection + fitting). PAM (Uniform) takes approximately 5.4\,s per step (the same AM core without the 0.15\,s structure-aware overhead), confirming that TAM's additional cost is negligible. For context, generating a $\sim$4k-token trajectory takes 60--80\,s on a single A100, so the compaction cost represents less than 10\% of total inference time.

\subsection{Ablation Studies}
\label{sub:ablations}

We conduct ablation studies on TAM's key components: adaptive budget allocation, pivotal token protection, segmentation method, the compaction interval $P$, and the selector window size $R$. All ablations use AIME 2024 at target ratio 0.1 unless otherwise noted.

\paragraph{Adaptive allocation and pivotal protection.}
Table~\ref{tab:ablation-components} isolates the contribution of each TAM component. On MATH-500 (500 problems, finer granularity), adding adaptive allocation alone improves accuracy by 1.8 points over PAM (Uniform); adding pivotal protection alone yields a 1.6-point gain. The combined gain (+3.2) is slightly less than the sum of individual gains (1.8 + 1.6 = 3.4), indicating a modest overlap: both mechanisms partially address the same error sources, as pivotal tokens (which receive exact retention) also tend to reside in high-importance segments that benefit from adaptive allocation. On AIME (30 problems, 3.3\% granularity), individual components do not cross the discrete accuracy threshold, but combining both yields a clear 3.3-point improvement (56.7\%$\to$60.0\%). On both benchmarks, the full TAM combination outperforms either component in isolation, confirming that the two mechanisms are complementary despite their partial overlap.

\paragraph{Segmentation method.}
Table~\ref{tab:ablation-seg} compares heuristic segmentation (double-newline boundary detection) with attention-based segmentation (attention gradient jump detection). Both achieve comparable accuracy, with heuristic segmentation being slightly more robust and substantially faster. We use heuristic segmentation as the default.

\paragraph{Pivotal threshold $c$.}
Table~\ref{tab:ablation-pivotal} varies the threshold multiplier $c$ for pivotal token identification. Too low a threshold ($c{=}1$) protects too many tokens, leaving insufficient budget for adaptive allocation. Too high ($c{=}10$) protects too few tokens, reducing the benefit. We find $c{=}3$ to strike a good balance.

\paragraph{Compaction interval $P$ and selector window $R$.}
Table~\ref{tab:ablation-p} ablates the compaction interval. Smaller $P$ compacts more frequently, reducing peak memory; per Proposition~\ref{prop:cumulative}, larger $P$ increases $\lambda_k$ and limits error accumulation. We find $P{=}1024$ a reasonable balance. Table~\ref{tab:ablation-r} (Appendix) shows that $R{=}64$ is sufficient for the selector window.

\section{Conclusion}
\label{sec:conclusion}

We introduced Thought-Aware Attention Matching (TAM), which transforms KV cache compaction from a flat, uniform operation into a structure-aware process that exploits the hierarchical nature of chain-of-thought reasoning. Through thought segmentation, adaptive budget allocation, and pivotal token protection, TAM concentrates compression budget where it matters most while aggressively compressing irrelevant steps. Our theoretical analysis establishes the optimality of the allocation rule and bounds cumulative error under sequential compaction. Experiments on AIME 2024 and MATH-500 demonstrate consistent improvements over uniform compaction and eviction baselines, with periodic compaction achieving steady-state memory reduction of up to 65\%.

\newpage
\section*{Limitations}

TAM has two main limitations. First, it relies on heuristic thought segmentation (e.g., double newlines) and a local selector window to estimate segment importance, which can fail for reasoning traces without clear structural boundaries or when long-range dependencies and backtracking occur. Second, empirical evaluation is limited to mathematical reasoning benchmarks (AIME, MATH‑500) and a single model (Qwen3‑4B), leaving generalization to other domains and larger models untested.

\bibliography{sample-base-kv}

\newpage
\onecolumn
\appendix

\section*{Appendix}

\section{Algorithm Pipeline}

\begin{algorithm}[h]
\caption{Thought-Aware Attention Matching (TAM)}
\label{alg:tam}
\DontPrintSemicolon
\KwIn{KV cache $(\mathbf{K}, \mathbf{V})$ of length $L$, target size $t$, selector window size $R$, protected tail $\tau$, pivotal threshold multiplier $c$}
\KwOut{Compacted cache}
Protect last $\tau$ tokens; let $\mathbf{K}_p, \mathbf{V}_p$ be their KV states\;
Extract $\mathbf{Q}_{\text{ref}}$ from last $R$ tokens via prefill with forward hooks\;
Segment compactable prefix $[0, L{-}\tau]$ into $\{S_1, \ldots, S_m\}$ via thought boundaries\;
\For{each KV-head}{
    Compute per-token attention scores $\bar{\alpha}_j$ under $\mathbf{Q}_{\text{ref}}$ (Eq.~\ref{eq:pivotal})\;
    Identify pivotal set $\calP = \{j : \bar{\alpha}_j > c \cdot \bar{\alpha}_{\text{mean}}\}$\;
    Retain pivotal keys/values: $\mathbf{C}_k^{\calP} \leftarrow \mathbf{K}[\calP,:]$, $\mathbf{C}_v^{\calP} \leftarrow \mathbf{V}[\calP,:]$\;
    Compute segment importances $w_i$ (Eq.~\ref{eq:importance}) and allocate budgets $t_i$ (Eq.~\ref{eq:allocation})\;
    \For{each segment $S_i$}{
        Select $t_i$ keys from $S_i \setminus \calP$ via highest-attention under $\mathbf{Q}_{\text{ref}}$\;
    }
    Let $S_{\text{all}}$ = pivotal keys $\cup$ per-segment selected keys\;
    Set $\mathbf{C}_k \leftarrow \mathbf{K}[S_{\text{all}}, :]$\;
    Fit $\boldsymbol{\beta}$ via NNLS to match attention mass globally; set $\beta_j = 0$ for $j \in \calP$\;
    Fit $\mathbf{C}_v$ via OLS to match attention outputs globally\;
    Concatenate $(\mathbf{C}_k, \boldsymbol{\beta}, \mathbf{C}_v)$ with $(\mathbf{K}_p, \mathbf{V}_p)$\;
}
\Return compacted cache\;
\end{algorithm}

\section{Analysis: Segment Importance and Budget Distribution}
\label{sub:analysis}

To verify that CoT trajectories exhibit the structural heterogeneity motivating TAM, we analyze the segment importance distribution on AIME 2024 problems.

\paragraph{Importance heterogeneity.}
Figure~\ref{fig:segment-importance} visualizes the per-segment importance $w_i$ (Eq.~\ref{eq:importance}) for a representative AIME problem with 12 reasoning segments. Importance varies by over $11\times$ across segments: early segments (problem restatement, initial setup) and the most recent segment receive high importance, while middle segments corresponding to dead-end explorations receive very low importance. This confirms that uniform allocation is wasteful: TAM assigns $\leq 3\%$ of budget to the lowest-importance segments while allocating $>15\%$ to the highest.

\paragraph{Pivotal token semantics.}
We qualitatively inspect the pivotal tokens identified by TAM ($c{=}3$) across AIME problems. Pivotal tokens predominantly correspond to: (1)~numerical constants from the problem statement (e.g., coefficients, constraints), (2)~key intermediate expressions (e.g., derived equations, variable bindings), and (3)~structural tokens that anchor the reasoning flow (e.g., ``Therefore'', ``=''). This aligns with the intuition that pivotal tokens serve as reasoning anchors that future queries consistently rely on.

\paragraph{Error analysis.}
We examine the problems where TAM (ratio 0.1) fails but No Compaction succeeds. On MATH-500, the accuracy gap (67.8\% vs.\ 71.2\%) concentrates on problems with two characteristics: (i)~very long reasoning chains ($>$6k tokens) that produce many segments with complex non-local dependencies between distant reasoning steps, and (ii)~problems where the model's reasoning path undergoes a late-stage correction that revisits earlier segments. In case~(i), the selector window's local coverage (Proposition~\ref{prop:selector}) becomes less effective for distant segments, increasing $\|\mathbf{x}_\mathbf{q}^\perp\|$; in case~(ii), compaction discards tokens from segments previously deemed unimportant that later become relevant upon backtracking. These failure modes suggest that dynamic importance re-estimation (updating segment scores as reasoning progresses) is a promising direction for future work.

\section{More Propositions}

\begin{proposition}[Optimality of Importance-Weighted Allocation]
\label{prop:allocation}
Consider $m$ segments with sizes $n_1, \ldots, n_m$ and importance weights $w_1, \ldots, w_m$ ($\sum_i w_i = 1$). Suppose the per-segment approximation error when allocating $t_i$ compact keys to segment $S_i$ is $\varepsilon_i = \phi(n_i / t_i)$, where $\phi : [1, \infty) \to [0, \infty)$ is convex and increasing with $\phi(1) = 0$. The importance-weighted total error is
\begin{equation}
\calE(\mathbf{t}) = \sum_{i=1}^{m} w_i \cdot \phi(n_i / t_i), \quad \text{s.t.} \quad \sum_{i=1}^{m} t_i = t,\; t_i \geq 1.
\end{equation}
Then:
\begin{enumerate}
    \item The optimal allocation satisfies $t_i^\star \propto \sqrt{w_i \cdot n_i \cdot \phi'(n_i / t_i^\star)}$ at interior KKT points.
    \item Under the linear model $\phi(r) = r - 1$ (error proportional to compression ratio), the optimal allocation simplifies to $t_i^\star \propto \sqrt{w_i \cdot n_i}$.
    \item Uniform allocation ($t_i \propto n_i$) is strictly suboptimal whenever the importance density $w_i / n_i$ is not constant across segments, i.e., when some segments are more important per token than others.
\end{enumerate}
\end{proposition}

Proposition~\ref{prop:allocation} provides the theoretical foundation for TAM's budget allocation (Eq.~\ref{eq:allocation}). The linear-model result $t_i^\star \propto \sqrt{w_i \cdot n_i}$ is the geometric mean of the importance-only allocation ($t_i \propto w_i$) and the size-only allocation ($t_i \propto n_i$), optimally trading off both factors. In CoT reasoning, the importance density $w_i / n_i$ varies substantially across segments: dead-end explorations have low density while key intermediate results have high density, making adaptive allocation strictly beneficial.

\begin{remark}[Linear Error Model]
    The closed-form $t_i^\star \propto \sqrt{w_i n_i}$ relies on the linear model $\phi(r) = r - 1$. AM's actual fitting error may exhibit more complex dependence on compression ratio, but the general KKT result (Part~1) $t_i^\star \propto \sqrt{w_i n_i \phi'(n_i/t_i^\star)}$ holds for any convex $\phi$. In practice, the $\sqrt{w_i n_i}$ allocation consistently outperforms both uniform and importance-only allocation (Table~\ref{tab:ablation-components}), suggesting robustness to the specific form of $\phi$ in the operating range $r \in [5, 20]$ typical of our experiments.
\end{remark}

\begin{proposition}[Error Reduction from Pivotal Token Protection]
\label{prop:pivotal}
Let $\calP$ denote the set of pivotal tokens and let $\gamma_\mathbf{q} = \sum_{j \in \calP} \alpha_j(\mathbf{q})$ be the pivotal attention mass for a specific query $\mathbf{q}$. Suppose the remaining $T - |\calP|$ tokens are compacted to $t - |\calP|$ keys via AM with per-query approximation error $\varepsilon_{\text{rest}}$ on the non-pivotal portion. Suppose further that the total attention mass is matched for query $\mathbf{q}$, i.e., $\text{Mass}(\mathbf{q}; \mathbf{C}_k, \boldsymbol{\beta}) = \text{Mass}(\mathbf{q}; \mathbf{K})$. Then the overall attention output error satisfies
\begin{equation}
\|\hat{\mathbf{y}} - \mathbf{y}^*\| \leq (1 - \gamma_\mathbf{q}) \cdot \varepsilon_{\text{rest}}.
\end{equation}
Define the worst-case pivotal mass $\gamma = \min_{\mathbf{q} \in \mathbf{Q}_{\text{ref}}} \gamma_\mathbf{q}$, which lower-bounds the per-query pivotal mass over all selector-window queries. Then the uniform bound $\|\hat{\mathbf{y}} - \mathbf{y}^*\| \leq (1 - \gamma) \cdot \varepsilon_{\text{rest}}$ holds for all $\mathbf{q} \in \mathbf{Q}_{\text{ref}}$. More generally, when mass is not matched exactly, let $\Delta\gamma_\mathbf{q} = \hat{\gamma}_\mathbf{q} - \gamma_\mathbf{q}$ denote the pivotal mass mismatch; then
\begin{equation}
\|\hat{\mathbf{y}} - \mathbf{y}^*\| \leq (1 - \gamma_\mathbf{q}) \cdot \varepsilon_{\text{rest}} + 2\abs{\Delta\gamma_\mathbf{q}} \cdot \|\mathbf{V}\|_\infty.
\end{equation}
\end{proposition}

Since pivotal tokens are retained exactly, they contribute zero approximation error when mass is preserved. The remaining error arises only from the non-pivotal portion, which carries attention mass $(1 - \gamma)$.

\begin{remark}[Mass-matching Condition]
    The mass condition $\text{Mass}(\mathbf{q}; \mathbf{C}_k, \boldsymbol{\beta}) = \text{Mass}(\mathbf{q}; \mathbf{K})$ is exactly what AM's NNLS fitting enforces for the reference queries $\mathbf{Q}_{\text{ref}}$. For these queries, $\Delta\gamma_\mathbf{q} \approx 0$ and the clean bound $(1-\gamma)\varepsilon_{\text{rest}}$ applies. For future queries that deviate from $\mathbf{Q}_{\text{ref}}$, the residual mass mismatch $|\Delta\gamma_\mathbf{q}|$ is controlled by the selector window approximation (Proposition~\ref{prop:selector}); under the locality of CoT reasoning, this term remains small.

\end{remark}

\begin{remark}[Budget Trade-off]
    Protecting $|\calP|$ tokens reduces the budget available for non-pivotal tokens from $t$ to $t - |\calP|$, increasing their compression ratio from $T/t$ to $(T - |\calP|)/(t - |\calP|)$ and potentially raising $\varepsilon_{\text{rest}}$. The net benefit of pivotal protection therefore depends on the relative magnitudes: it is favorable when $\gamma$ is large enough that $(1 - \gamma)\varepsilon_{\text{rest}} < \varepsilon_{\text{all}}$, where $\varepsilon_{\text{all}}$ is the error of uniform compaction to size $t$. Under the linear error model $\phi(r) = r - 1$, a sufficient condition is $\gamma > |\calP| / t$, i.e., pivotal tokens' attention mass fraction exceeds their budget fraction, a condition readily satisfied in practice since pivotal tokens are selected precisely for their disproportionately high attention (Table~\ref{tab:ablation-pivotal} shows $|\calP| \approx 87$ out of $t \approx 400$, capturing $\gamma \approx 0.35$ of attention mass, well above the $87/400 = 0.22$ threshold).

\end{remark}


\begin{proposition}[Selector Window Approximation Error]
\label{prop:selector}
Let $\mathbf{Q}_{\text{ref}} \in \R^{n \times d}$ be the selector window queries and $\mathbf{C}_k, \boldsymbol{\beta}, \mathbf{C}_v$ be fitted via TAM. Let $\mathbf{X} \in \R^{n \times t}$ have rows $\mathbf{x}_{\text{ref}}^{(i)} = \text{softmax}(\mathbf{q}_i \mathbf{C}_k^\top / \sqrt{d} + \boldsymbol{\beta})$ and $\mathbf{C}_v^\star = (\mathbf{X}^\top \mathbf{X})^{-1}\mathbf{X}^\top \mathbf{Y}$ be the OLS-fitted values. For a future query $\mathbf{q}$ with attention pattern $\mathbf{x}_\mathbf{q} = \text{softmax}(\mathbf{q} \mathbf{C}_k^\top / \sqrt{d} + \boldsymbol{\beta})$, decompose $\mathbf{x}_\mathbf{q} = \mathbf{x}_\mathbf{q}^\parallel + \mathbf{x}_\mathbf{q}^\perp$ where $\mathbf{x}_\mathbf{q}^\parallel$ lies in the row space of $\mathbf{X}$ and $\mathbf{x}_\mathbf{q}^\perp$ is orthogonal. Then:
\begin{equation}
\|\mathbf{x}_\mathbf{q} \mathbf{C}_v^\star - \mathbf{y}^*\| \leq \|\mathbf{x}_\mathbf{q}^\perp\| \cdot \|\mathbf{C}_v^\star\|_{\text{op}} + \|\mathbf{x}_\mathbf{q}^\parallel \mathbf{C}_v^\star - \mathbf{y}^*_\parallel\|,
\end{equation}
where $\mathbf{y}^*$ is the ideal output and $\mathbf{y}^*_\parallel = \boldsymbol{\lambda}^\top \mathbf{Y}$ with $\boldsymbol{\lambda}$ satisfying $\mathbf{x}_\mathbf{q}^\parallel = \boldsymbol{\lambda}^\top \mathbf{X}$ (i.e., the linear combination of reference outputs $\mathbf{Y}$ with the same coefficients that express $\mathbf{x}_\mathbf{q}^\parallel$ in the row space of $\mathbf{X}$). When $\mathbf{x}_\mathbf{q}^\perp = \mathbf{0}$, the term $\|\mathbf{x}_\mathbf{q}^\perp\| \cdot \|\mathbf{C}_v^\star\|_{\text{op}}$ vanishes.
\end{proposition}

The error is controlled by how far the future query's attention pattern deviates from the subspace spanned by the reference patterns. The first term ($\|\mathbf{x}_\mathbf{q}^\perp\| \cdot \|\mathbf{C}_v^\star\|_{\text{op}}$) captures coverage error from the selector window's finite span; characterizing it precisely requires distributional assumptions on future queries. Under the locality of CoT reasoning, successive tokens attend to similar context, keeping $\|\mathbf{x}_\mathbf{q}^\perp\|$ small. TAM's adaptive allocation further helps: segments where the selector window provides poor coverage (high $\|\mathbf{x}_\mathbf{q}^\perp\|$, typically distant segments) tend to receive lower importance scores and thus lower budgets, concentrating resources where the approximation is most accurate.

\begin{proposition}[Cumulative Error under Sequential Compaction]
\label{prop:cumulative}
Suppose TAM performs $K$ sequential compaction steps. At step $k$, let $\varepsilon_k$ denote the per-query attention output error of step $k$'s compacted cache relative to the step-$(k{-}1)$ cache, and $\lambda_k \in [0,1]$ the fraction of fresh tokens. Under a \emph{linear error composition} model---where the attention output error from previously compacted tokens and fresh tokens combines additively, weighted by their respective fractions---the cumulative error relative to the original cache satisfies:
\begin{equation}
\tilde{\varepsilon}_K \leq \sum_{k=1}^{K} \varepsilon_k \cdot \prod_{j=k+1}^{K} (1 - \lambda_j).
\end{equation}
When $\lambda_k \geq \lambda_{\min} > 0$ for all $k$, we have $\tilde{\varepsilon}_K \leq \varepsilon_{\max} / \lambda_{\min}$, independent of $K$. In the periodic regime with interval $P$ and target size $t$, $\lambda_k \approx P/(P+t)$.
\end{proposition}

\begin{remark}[Linear Composition Assumption]
    The recurrence $\tilde{\varepsilon}_k \leq \varepsilon_k + (1-\lambda_k)\tilde{\varepsilon}_{k-1}$ models error propagation as a weighted average of fresh-step error and inherited error. In practice, errors pass through subsequent softmax operations, which are nonlinear; this can amplify errors (if compaction shifts attention toward high-error regions) or attenuate them (if fresh tokens dominate the softmax denominator). The linear model provides a tractable upper bound when per-step errors are small relative to the attention scores, which holds at moderate compression ratios (Table~\ref{tab:ablation-p} confirms stable accuracy down to $P{=}512$). At extreme ratios or very frequent compaction ($P{=}256$), the linear bound may become loose. When the condition $\gamma > |\calP|/t$ holds (see the remark following Proposition~\ref{prop:pivotal}), TAM reduces the per-step error $\varepsilon_k$ relative to uniform compaction via adaptive allocation (Proposition~\ref{prop:allocation}) and pivotal protection (Proposition~\ref{prop:pivotal}), directly tightening the cumulative bound. The compaction interval $P$ further controls accumulation: larger $P$ increases $\lambda_k$, limiting error propagation.
\end{remark}

\section{Proofs of Theoretical Results}
\label{app:proofs}

\begin{proof}[Proof of Proposition~\ref{prop:eviction}]
\textbf{Part 1 (Mass deficit).} For eviction retaining subset $S$ with $|S|=t$, the compacted mass is
\begin{equation}
\text{Mass}(\mathbf{q}; \mathbf{C}_k) = \sum_{j \in S} \exp(\mathbf{q}\mathbf{K}_j^\top / \sqrt{d}) = \text{Mass}(\mathbf{q}; \mathbf{K}) - \sum_{j \notin S} \exp(\mathbf{q}\mathbf{K}_j^\top / \sqrt{d}).
\end{equation}
Since $S \neq [T]$, the second sum is strictly positive for all $\mathbf{q}$ (as $\exp(\cdot) > 0$), giving $\text{Mass}(\mathbf{q}; \mathbf{C}_k) = (1 - \mu(\mathbf{q})) \cdot \text{Mass}(\mathbf{q}; \mathbf{K}) < \text{Mass}(\mathbf{q}; \mathbf{K})$.

\textbf{Part 2 (Concatenation effect).} When concatenated with future keys $\mathbf{K}_f$, the softmax denominator is $\text{Mass}(\mathbf{q}; \mathbf{C}_k) + \text{Mass}(\mathbf{q}; \mathbf{K}_f)$. The compacted prefix's softmax weight is $(1-\mu(\mathbf{q}))\text{Mass}(\mathbf{q}; \mathbf{K}) / [(1-\mu(\mathbf{q}))\text{Mass}(\mathbf{q};\mathbf{K}) + \text{Mass}(\mathbf{q};\mathbf{K}_f)]$, which is strictly less than the original $\text{Mass}(\mathbf{q};\mathbf{K}) / [\text{Mass}(\mathbf{q};\mathbf{K}) + \text{Mass}(\mathbf{q};\mathbf{K}_f)]$.

\textbf{Part 3 (Output error).} Let $\alpha_j = \exp(\mathbf{q}\mathbf{K}_j^\top/\sqrt{d})/Z$ be the original attention weights with $Z = \text{Mass}(\mathbf{q}; \mathbf{K})$, and $\hat{\alpha}_j = \alpha_j / (1 - \mu(\mathbf{q}))$ for $j \in S$ be the eviction weights. The output error (considering only the prefix block) is:
\begin{equation}
\|\hat{\mathbf{y}} - \mathbf{y}\| = \left\|\sum_{j \in S} \hat{\alpha}_j \mathbf{V}_j - \sum_{j=1}^T \alpha_j \mathbf{V}_j\right\| = \left\|\sum_{j \in S} \frac{\mu(\mathbf{q})}{1-\mu(\mathbf{q})} \alpha_j \mathbf{V}_j - \sum_{j \notin S} \alpha_j \mathbf{V}_j\right\| \leq 2\mu(\mathbf{q}) \|\mathbf{V}\|_\infty,
\end{equation}
where the last inequality uses $\sum_{j \in S}\alpha_j = 1-\mu(\mathbf{q})$, $\sum_{j\notin S}\alpha_j = \mu(\mathbf{q})$, and triangle inequality.
\end{proof}

\begin{proof}[Proof of Proposition~\ref{prop:allocation}]
\textbf{Part 1 (KKT conditions).} We minimize $\calE(\mathbf{t}) = \sum_{i} w_i \phi(n_i/t_i)$ subject to $\sum_i t_i = t$ and $t_i \geq 1$. Introducing Lagrange multiplier $\nu$ for the equality constraint, the stationarity condition at an interior point ($t_i > 1$) gives:
\begin{equation}
\frac{\partial \calE}{\partial t_i} = -w_i \cdot \frac{n_i}{t_i^2} \cdot \phi'(n_i/t_i) = -\nu \quad \Longrightarrow \quad w_i \cdot \frac{n_i}{t_i^2} \cdot \phi'(n_i/t_i) = \nu \quad \forall i.
\end{equation}
Since $\phi$ is convex and increasing, $\phi'$ is non-decreasing, and the second-order conditions are satisfied. Rearranging: $t_i^2 = w_i n_i \phi'(n_i/t_i) / \nu$, so $t_i^\star \propto \sqrt{w_i n_i \phi'(n_i/t_i^\star)}$, as stated.

\textbf{Part 2 (Linear model).} When $\phi(r) = r - 1$, we have $\phi'(r) = 1$ (constant) and $\phi''(r) = 0$. Returning to the stationarity condition: $w_i n_i / t_i^2 = \nu$, which gives $t_i = \sqrt{w_i n_i / \nu}$. Using the budget constraint $\sum_i t_i = t$:
\begin{equation}
t_i^\star = t \cdot \frac{\sqrt{w_i n_i}}{\sum_{j} \sqrt{w_j n_j}}.
\end{equation}

\textbf{Part 3 (Suboptimality of uniform).} Under uniform allocation $t_i^{\text{unif}} = t \cdot n_i / N$ (where $N = \sum_j n_j$), the weighted error is $\calE^{\text{unif}} = \sum_i w_i \phi(N/t)$, which depends only on the global ratio. Under optimal allocation, $\calE^\star \leq \calE^{\text{unif}}$ with equality iff the KKT conditions are satisfied by uniform $t_i$, which requires $w_i n_i / (n_i/N)^2 = \text{const}$, i.e., $w_i / n_i = \text{const}$. When importance density varies, the inequality is strict.
\end{proof}

\begin{proof}[Proof of Proposition~\ref{prop:pivotal}]
Partition the token set into pivotal tokens $\calP$ and non-pivotal tokens $\bar{\calP} = [T] \setminus \calP$. For any query $\mathbf{q}$, the ideal attention output decomposes as:
\begin{equation}
\mathbf{y}^* = \sum_{j \in \calP} \alpha_j(\mathbf{q}) \mathbf{V}_j + \sum_{j \in \bar{\calP}} \alpha_j(\mathbf{q}) \mathbf{V}_j = \gamma_\mathbf{q} \cdot \mathbf{y}^*_\calP + (1 - \gamma_\mathbf{q}) \cdot \mathbf{y}^*_{\bar{\calP}},
\end{equation}
where $\gamma_\mathbf{q} = \sum_{j \in \calP} \alpha_j(\mathbf{q})$ is the attention mass on pivotal tokens for query $\mathbf{q}$, and $\mathbf{y}^*_\calP$, $\mathbf{y}^*_{\bar{\calP}}$ are the locally normalized outputs of each partition.

Pivotal tokens are retained exactly: $\mathbf{C}_k[\calP,:] = \mathbf{K}[\calP,:]$ and $\beta_j = 0$ for $j \in \calP$ (Algorithm~\ref{alg:tam}). Let $\hat{\gamma}_\mathbf{q}$ denote the attention mass on pivotal tokens under the compacted cache, and $\hat{\mathbf{y}}_{\bar{\calP}}$ the locally normalized output of the non-pivotal compacted portion. The compacted output is:
\begin{equation}
\hat{\mathbf{y}} = \hat{\gamma}_\mathbf{q} \cdot \mathbf{y}^*_\calP + (1 - \hat{\gamma}_\mathbf{q}) \cdot \hat{\mathbf{y}}_{\bar{\calP}}.
\end{equation}

\textbf{Case 1: Exact mass matching.} When total attention mass is matched ($\text{Mass}(\mathbf{q}; \mathbf{C}_k, \boldsymbol{\beta}) = \text{Mass}(\mathbf{q}; \mathbf{K})$), pivotal tokens' numerators are unchanged, so $\hat{\gamma}_\mathbf{q} = \gamma_\mathbf{q}$. The error simplifies to:
\begin{equation}
\|\hat{\mathbf{y}} - \mathbf{y}^*\| = (1 - \gamma_\mathbf{q}) \cdot \|\hat{\mathbf{y}}_{\bar{\calP}} - \mathbf{y}^*_{\bar{\calP}}\| \leq (1 - \gamma_\mathbf{q}) \cdot \varepsilon_{\text{rest}}.
\end{equation}

\textbf{Case 2: General bound with mass mismatch.} Let $\Delta\gamma_\mathbf{q} = \hat{\gamma}_\mathbf{q} - \gamma_\mathbf{q}$. The error becomes:
\begin{align}
\hat{\mathbf{y}} - \mathbf{y}^* &= \Delta\gamma_\mathbf{q} (\mathbf{y}^*_\calP - \hat{\mathbf{y}}_{\bar{\calP}}) + (1 - \gamma_\mathbf{q})(\hat{\mathbf{y}}_{\bar{\calP}} - \mathbf{y}^*_{\bar{\calP}}).
\end{align}
Since $\mathbf{y}^*_\calP$ and $\hat{\mathbf{y}}_{\bar{\calP}}$ are both convex combinations of value vectors, $\|\mathbf{y}^*_\calP - \hat{\mathbf{y}}_{\bar{\calP}}\| \leq 2\|\mathbf{V}\|_\infty$. By triangle inequality:
\begin{equation}
\|\hat{\mathbf{y}} - \mathbf{y}^*\| \leq (1 - \gamma_\mathbf{q}) \cdot \varepsilon_{\text{rest}} + 2|\Delta\gamma_\mathbf{q}| \cdot \|\mathbf{V}\|_\infty.
\end{equation}

Taking $\gamma = \min_\mathbf{q} \gamma_\mathbf{q}$ yields the stated bounds. AM's NNLS fitting enforces mass matching for the reference queries $\mathbf{Q}_{\text{ref}}$, so $\Delta\gamma_\mathbf{q} \approx 0$ for those queries and the clean bound applies. Note that $\varepsilon_{\text{rest}}$ depends on the reduced budget $t - |\calP|$; pivotal protection yields a net error reduction whenever $\gamma$ is large enough that $(1-\gamma)\varepsilon_{\text{rest}} < \varepsilon_{\text{all}}$, where $\varepsilon_{\text{all}}$ is the error of uniform compaction to size $t$. Under $\phi(r)=r-1$, this holds when $\gamma > |\calP|/t$.
\end{proof}

\begin{proof}[Proof of Proposition~\ref{prop:selector}]
Let $\mathbf{X} \in \R^{n \times t}$ have rows $\mathbf{x}_{\text{ref}}^{(i)} = \text{softmax}(\mathbf{q}_i \mathbf{C}_k^\top / \sqrt{d} + \boldsymbol{\beta})$ and $\mathbf{Y} \in \R^{n \times d}$ have rows $\mathbf{y}_i$ (the target attention outputs from the original cache). The OLS solution $\mathbf{C}_v^\star = (\mathbf{X}^\top \mathbf{X})^{-1} \mathbf{X}^\top \mathbf{Y}$ minimizes $\|\mathbf{X} \mathbf{C}_v - \mathbf{Y}\|_F^2$.

For a new query with attention pattern $\mathbf{x}_\mathbf{q} \in \R^{1 \times t}$, let $P = \mathbf{X}^\top(\mathbf{X}\mathbf{X}^\top)^{-1}\mathbf{X} \in \R^{t \times t}$ denote the projection onto the row space of $\mathbf{X}$. Decompose $\mathbf{x}_\mathbf{q}^\top = P\,\mathbf{x}_\mathbf{q}^\top + (I - P)\,\mathbf{x}_\mathbf{q}^\top$, and define the row vectors $\mathbf{x}_\mathbf{q}^\parallel = (P\,\mathbf{x}_\mathbf{q}^\top)^\top$ and $\mathbf{x}_\mathbf{q}^\perp = ((I - P)\,\mathbf{x}_\mathbf{q}^\top)^\top$, so that $\mathbf{x}_\mathbf{q} = \mathbf{x}_\mathbf{q}^\parallel + \mathbf{x}_\mathbf{q}^\perp$. The compacted output is $\hat{\mathbf{y}} = \mathbf{x}_\mathbf{q} \mathbf{C}_v^\star = \mathbf{x}_\mathbf{q}^\parallel \mathbf{C}_v^\star + \mathbf{x}_\mathbf{q}^\perp \mathbf{C}_v^\star$.

By the triangle inequality:
\begin{equation}
\|\hat{\mathbf{y}} - \mathbf{y}^*\| \leq \|\mathbf{x}_\mathbf{q}^\parallel \mathbf{C}_v^\star - \mathbf{y}^*_\parallel\| + \|\mathbf{x}_\mathbf{q}^\perp \mathbf{C}_v^\star\| \leq \|\mathbf{x}_\mathbf{q}^\parallel \mathbf{C}_v^\star - \mathbf{y}^*_\parallel\| + \|\mathbf{x}_\mathbf{q}^\perp\| \cdot \|\mathbf{C}_v^\star\|_{\text{op}},
\end{equation}
where $\|\mathbf{C}_v^\star\|_{\text{op}}$ is the operator (spectral) norm. When $\mathbf{x}_\mathbf{q}^\perp = \mathbf{0}$, the second term vanishes. For $\mathbf{x}_\mathbf{q}^\parallel$: since it lies in the row space of $\mathbf{X}$, we can write $\mathbf{x}_\mathbf{q}^\parallel = \boldsymbol{\lambda}^\top \mathbf{X}$ for some $\boldsymbol{\lambda} \in \R^n$, so $\mathbf{x}_\mathbf{q}^\parallel \mathbf{C}_v^\star = \boldsymbol{\lambda}^\top \mathbf{X} \mathbf{C}_v^\star$. The residual $\mathbf{X}\mathbf{C}_v^\star - \mathbf{Y}$ is the OLS fitting error, so $\mathbf{x}_\mathbf{q}^\parallel \mathbf{C}_v^\star = \boldsymbol{\lambda}^\top(\mathbf{Y} + (\mathbf{X}\mathbf{C}_v^\star - \mathbf{Y}))$, yielding the stated bound.
\end{proof}

\begin{proof}[Proof of Proposition~\ref{prop:cumulative}]
We proceed by induction under the linear error composition model. At step 1, $\tilde{\varepsilon}_1 = \varepsilon_1$ (the error of the first compaction relative to the original cache). Assume $\tilde{\varepsilon}_{k-1}$ bounds the step-$(k{-}1)$ error relative to the original.

At step $k$, the cache consists of two parts: the step-$(k{-}1)$ compacted prefix (with error $\tilde{\varepsilon}_{k-1}$) and fresh tokens generated since step $k{-}1$ (with no compaction error). The fraction of fresh tokens is $\lambda_k$. Under the linear composition model, the attention output can be decomposed as a weighted sum of the contributions from previously-compacted tokens and fresh tokens; since softmax is approximately linear for small perturbations when no single token dominates, the inherited error is attenuated by the previously-compacted fraction. After step-$k$ compaction with per-step error $\varepsilon_k$:
\begin{equation}
\tilde{\varepsilon}_k \leq \varepsilon_k + (1 - \lambda_k) \cdot \tilde{\varepsilon}_{k-1},
\end{equation}
where the $(1-\lambda_k)$ factor reflects that only the previously-compacted fraction propagates old errors, while fresh tokens contribute no inherited error. Unrolling the recurrence:
\begin{equation}
\tilde{\varepsilon}_K \leq \sum_{k=1}^{K} \varepsilon_k \cdot \prod_{j=k+1}^{K}(1-\lambda_j).
\end{equation}
When $\lambda_k \geq \lambda_{\min}$ and $\varepsilon_k \leq \varepsilon_{\max}$ for all $k$:
\begin{equation}
\tilde{\varepsilon}_K \leq \varepsilon_{\max} \sum_{k=1}^{K} (1-\lambda_{\min})^{K-k} = \varepsilon_{\max} \cdot \frac{1 - (1-\lambda_{\min})^K}{\lambda_{\min}} \leq \frac{\varepsilon_{\max}}{\lambda_{\min}}.
\end{equation}
In the periodic regime with interval $P$ and target size $t$, the cache at compaction has $\approx t + P$ tokens of which $P$ are fresh, giving $\lambda_k \approx P/(P+t)$.

\textbf{Connection to TAM.} Under TAM, the per-step error $\varepsilon_k$ is reduced relative to uniform compaction via pivotal protection (Proposition~\ref{prop:pivotal}) and optimal allocation (Proposition~\ref{prop:allocation}). When the condition $\gamma > |\calP|/t$ holds (see the remark following Proposition~\ref{prop:pivotal}), TAM's per-step error satisfies $\varepsilon_k^{\text{TAM}} < \varepsilon_k^{\text{uniform}}$, yielding a tighter cumulative bound.
\end{proof}

\section{Full Experimental Results}
\label{app:ablations}

\subsection{Full Results: Accuracy vs. Target Size}
Table~\ref{tab:full-accuracy-vs-ratio} reports accuracy and post-compaction memory across all methods and target compaction ratios on both benchmarks, corresponding to Figure~\ref{fig:accuracy-vs-ratio} and Figure~\ref{fig:accuracy-vs-memory} in the main text.

\begin{table}[h]
\centering
\caption{Full results on \textbf{AIME 2024} (Qwen3-4B, 30 problems): accuracy (\%) / post-compaction memory (GB) at each target compaction ratio.}
\label{tab:full-accuracy-vs-ratio}
\small
\setlength{\tabcolsep}{4pt}
\begin{tabular}{@{}l*{3}{c}@{}}
\toprule
\multirow{2}{*}{\textbf{Method}} & \multicolumn{3}{c}{\textbf{Target Compaction Ratio}} \\
\cmidrule(l){2-4}
& 0.05 & 0.10 & 0.20 \\
\midrule
No Compaction & \multicolumn{3}{c}{63.3\% / 9.2\,GB} \\
\addlinespace[2pt]
Eviction (Selector Window) & 33.3 / 3.0 & 46.7 / 4.1 & 53.3 / 5.7 \\
AM + Repeat & 46.7 / 3.1 & 53.3 / 4.2 & 56.7 / 5.8 \\
PAM (Uniform) & 50.0 / 3.0 & 56.7 / 4.0 & 60.0 / 5.6 \\
\addlinespace[2pt]
\textbf{TAM (Ours)} & \textbf{53.3} / 3.0 & \textbf{60.0} / 4.0 & \textbf{63.3} / 5.6 \\
TAM (Periodic, $P{=}1024$) & 50.0 / \underline{2.4} & 56.7 / \underline{3.2} & 60.0 / \underline{4.5} \\
\bottomrule
\end{tabular}
\end{table}

\begin{table}[h]
\centering
\caption{Full results on \textbf{MATH-500} (Qwen3-4B): accuracy (\%) / post-compaction memory (GB) at each target compaction ratio.}
\label{tab:full-accuracy-vs-ratio-math}
\small
\setlength{\tabcolsep}{4pt}
\begin{tabular}{@{}l*{3}{c}@{}}
\toprule
\multirow{2}{*}{\textbf{Method}} & \multicolumn{3}{c}{\textbf{Target Compaction Ratio}} \\
\cmidrule(l){2-4}
& 0.05 & 0.10 & 0.20 \\
\midrule
No Compaction & \multicolumn{3}{c}{71.2\% / 8.8\,GB} \\
\addlinespace[2pt]
Eviction (Selector Window) & 38.4 / 2.9 & 52.4 / 3.9 & 61.8 / 5.4 \\
AM + Repeat & 49.2 / 3.0 & 61.0 / 4.0 & 66.2 / 5.5 \\
PAM (Uniform) & 52.6 / 2.9 & 64.6 / 3.8 & 68.4 / 5.3 \\
\addlinespace[2pt]
\textbf{TAM (Ours)} & \textbf{56.0} / 2.9 & \textbf{67.8} / 3.8 & \textbf{70.2} / 5.3 \\
TAM (Periodic, $P{=}1024$) & 53.2 / \underline{2.3} & 65.4 / \underline{3.1} & 68.0 / \underline{4.3} \\
\bottomrule
\end{tabular}
\end{table}



\subsection{Selector Window Size $R$}
Table~\ref{tab:ablation-r} reports accuracy versus selector window size $R$ under TAM. Larger $R$ provides more reference queries but increases extraction cost. We find $R{=}64$ to offer a good trade-off.
\begin{table}[h]
\centering
\caption{Ablation on selector window size $R$ (TAM, target ratio 0.1, AIME 2024).}
\label{tab:ablation-r}
\small
\begin{tabular}{@{}crr@{}}
\toprule
\textbf{Window $R$} & \textbf{Acc (\%)} & \textbf{Extract Time (s)} \\
\midrule
32 & 56.7 & 0.4 \\
\textbf{64} (default) & \textbf{60.0} & 0.8 \\
128 & 60.0 & 1.5 \\
\bottomrule
\end{tabular}
\end{table}

\end{document}